\documentclass[german,english,a4paper,11pt]{article}

\usepackage{babel}
\usepackage{amsmath, amsxtra, amsfonts, amsthm, amssymb}
\usepackage[noadjust]{cite}
\usepackage{url}
\usepackage{graphicx,color}
\usepackage[colorlinks]{hyperref}
\definecolor{linkblue}{rgb}{0.1,0.1,0.8}
\hypersetup{colorlinks=true,linkcolor=linkblue,filecolor=linkblue,urlcolor=linkblue,citecolor=linkblue}
\usepackage[algo2e,ruled,vlined,linesnumbered]{algorithm2e}

\newtheorem{theorem}{Theorem}
\newtheorem{lemma}[theorem]{Lemma}

\newtheorem{corollary}[theorem]{Corollary}

\newcommand{\ignore}[1]{}

\renewcommand{\O}{\ensuremath{\mathord{O}}} %% Big-Oh-Notation
\renewcommand{\o}{\ensuremath{\mathord{o}}} %% Small-oh-Notation

\renewcommand{\Pr}{\ensuremath{\mathrm{Pr}}}
\newcommand{\EXP}{\ensuremath{\operatorname{E}}}

\newcommand{\euler}{\mathrm{e}\xspace}

\newcommand{\onemax}{\textsc{OneMax}\xspace}
\newcommand{\binval}{\textsc{BinVal}\xspace}

\newcommand{\bmid}{\,\big|\,}

\title{Multiplicative Drift Analysis}
\author{
\selectlanguage{german}
Benjamin Doerr \and Daniel Johannsen \and Carola Winzen\thanks{Carola Winzen is a recipient of the Google Europe Fellowship in Randomized Algorithms, and this work is supported in part by this Google Fellowship.}\\
Max-Planck-Institut f\"{u}r Informatik\\
Campus E1 4 \\
66123 Saarbr\"{u}cken, Germany\\
}%
\date{Submitted January 2011}
\begin{document}
\maketitle

\begin{abstract}
In this work, we introduce multiplicative drift analysis as a suitable way to analyze the runtime of randomized search heuristics such as evolutionary algorithms.

We give a multiplicative version of the classical drift theorem. This allows easier analyses in those settings where the optimization progress is roughly proportional to the current distance to the optimum.

To display the strength of this tool, we regard the classical problem how the (1+1)~Evolutionary Algorithm optimizes an arbitrary linear pseudo-Boolean function. Here, we first give a relatively simple proof for the fact that any linear function is optimized in expected time $O(n \log n)$, where $n$ is the length of the bit string. Afterwards, we show that in fact any such function is optimized in expected time at most ${(1+o(1)) 1.39 \euler n\ln (n)}$, again using multiplicative drift analysis. We also prove a corresponding lower bound of ${(1-o(1))e n\ln(n)}$ which actually holds for all functions with a unique global optimum.

We further demonstrate how our drift theorem immediately gives natural proofs (with better constants) for the best known runtime bounds for the (1+1)~Evolutionary Algorithm on combinatorial problems like finding minimum spanning trees, shortest paths, or Euler tours.
\end{abstract}

\newpage

\section{Introduction}
\label{sec:introduction}

An innocent looking problem is the question how long the (1+1)~Evolutionary Algorithm ((1+1)~EA) needs to find the optimum of a given linear function. However, this is in fact one of the problems that was most influential for the theory of evolutionary algorithms. 

While particular linear functions like the functions \onemax and \binval were easily analyzed, it took a major effort by Droste, Jansen and Wegener~\cite{DrosteJW02} to solve the problem in full generality. Their proof, however, is highly technical. 

A major breakthrough spurred by this problem is the work by He and Yao~\cite{HeY01,HeY04}, who introduced \emph{drift analysis} to the field of evolutionary computation. This allowed a significantly simpler proof for the linear functions problem. Even more important, it quickly became one of the most powerful tools for both proving upper and lower bounds on the expected optimization times of evolutionary algorithms. For example, see~\cite{HeY04, GielWstacs03, GielL06, HappJKN08, NeumannOW09, OlivetoW10}.

Another great progress was made by J\"agersk\"upper~\cite{Jagerskupper08}, who combined drift analysis with a clever averaging argument to determine reasonable values for the usually not explicitly given constants. More precisely, J\"agersk\"upper showed that the expected optimization time of the (1+1)~EA for any linear function defined on bit strings of length~$n$ is bounded from above by $(1+o(1)) 2.02 \euler n \ln(n)$.
\subsection{Classical Drift Analysis}

The following method was introduced to the analysis of randomized search heuristics by He and Yao~\cite{HeY04} and builds on a result of Hajek~\cite{Hajek82}. When analyzing the optimization behavior of a randomized search heuristic over a search space, instead of tracking how the objective function improves, one uses an auxiliary potential function and tracks its behavior.

For example, consider the search space~$\{0,1\}^n$ of bitstrings of length $n\in\mathbb{N}$.\footnote{By~$\mathbb{N}:=\{0,1,2,\dots\}$ we denote the set of integers including zero and by~$\mathbb{R}$ we denote the set of real numbers.} Suppose we want to analyze the (1+1)~EA (which is introduced as Algorithm~\ref{alg:oneoneea} in Section~\ref{sec:linear}) minimizing a linear function $f\colon\{0,1\}^n\to\mathbb{R}$ with
\begin{equation*}
f(x)=\sum_{i=1}^n w_i x_i
\end{equation*}
and arbitrary positive weights $0<w_1\le\dots\le w_n$. (Note that we differ from previous works by always considering minimization problems. See Section~\ref{sec:linear} for a discussion why this does not influence the runtime analysis.) Then this potential function $h\colon \{0,1\}^n \to \mathbb{R}$ can be chosen as
\begin{equation}
\label{eq:logpotential}
h(x)=\ln\Big(1+\sum_{j=1}^{\lfloor \frac{n}{2}\rfloor}x_j+\sum_{j=\lfloor \frac{n}{2}\rfloor +1}^n 2x_j\Big).
\end{equation}

Though still needing some calculations, one can show the following (see, e.g.,~\cite{HeY04} where a variant of Algorithm~\ref{alg:oneoneea} is analyzed). Let~$x \in \{0,1\}^n$. Let~$y \in \{0,1\}^n$ be the result of one iteration (mutation and selection) of the (1+1)~EA started in $x$. Then there exists a $\delta > 0$ such that
\begin{equation}
\label{eq:drift}
 \EXP[h(y)] \le h(x) - \delta/n.
\end{equation}
Now, classical drift analysis tells us that in expectation after a number of $h(x)/(\delta/n) = O(n \log n)$ iterations, the potential value is reduced to zero. But $h(x)=0$ implies $f(x)=0$, that is, the (1+1)~EA has found the desired optimum.

%It turns out that, up to a multiplicative factor, the potential function~$h$ approximates the expected optimization time of the (1+1)~EA starting in the respective point up to lower order terms. This is no coincidence. In fact, in Section~\ref{sec:drift}, we show that for every objective function, we can build a well-suited potential function by assigning to each point of the search space the expected optimization time of the (1+1)~EA starting in that point.

Using drift analysis to analyze a randomized search heuristic usually bears two difficulties. The first is guessing a suitable potential function~$h$. The second, related to the first, is proving that during the search, $f$ and~$h$ behave sufficiently similar, that is, we can prove some statement like inequality~(\ref{eq:drift}). Note that this inequality contains information about $f$ as well, namely implicitly in the fact that $y$ has an at least as good $f$--value as~$x$.

A main difficulty in showing that $h$ in~(\ref{eq:logpotential}) is a suitable potential function is the logarithm around the simple linear function giving weights one and two to the bits. However, since the optimization progress for linear functions is faster if we are further away from the optimum, that is, have more one-bits, this seems difficult to avoid.

\subsection{Multiplicative Drift Analysis}

We present a way to ease to use of drift analysis in such settings. Informally, our method applies if we have a potential function $g$ satisfying
\begin{equation}
\label{eq:drift2}
 \EXP[g(y)] \le (1 - \delta)g(x)
\end{equation}
in the notation above. That is, we require a progress which \emph{multiplicatively} depends on the current potential value. For this reason we call the method \emph{multiplicative drift analysis}. We will see that for a number of problems such potential functions are a natural choice. 

This new method allows us to largely separate the structural analysis of an optimization process from the actual calculation of a bound on the expected optimization time. Moreover, the runtime bounds obtained by multiplicative drift analysis are often sharper than those resulting from previously used techniques.

\subsection{Our Results}

We apply this new tool, multiplicative drift analysis, to the already mentioned problem of optimizing linear functions over~$\{0,1\}^n$. This yields a simplified proof of the $O(n\log n)$ bound on the optimization time of the (1+1)~EA. Similar to the proof using the classical drift theorem, we make use of the simple linear function $g\colon \{0,1\}^n \to \mathbb{R}$, chosen as
\begin{equation*}
g(x)=\sum_{i=1}^n \Big(1+\frac{i}{n}\Big)x_i.
\end{equation*}
This function~$g$ serves us as a potential function for all linear functions $f\colon \{0,1\}^n \to \mathbb{R}$ with
\begin{equation*}
f(x)=\sum_{i=1}^n w_i x_i
\end{equation*}
and monotone weights~$0<w_1\le\dots \le w_n$.
% Here, the beauty of the results lies in the fact that it shows the ease we gain by applying multiplicative drift analysis.

Using parts of J\"agersk\"upper's analysis~\cite{Jagerskupper08}, we then improve his upper bound on the expected optimization time of the (1+1)~EA on linear functions to $(1+o(1)) 1.39 \euler n \ln(n)$. 

We also give lower bounds for this problem. We show that, in the class of all functions with a unique global optimum, the function \onemax (see~(\ref{eq:onemax})) has the smallest expected optimization time. This extends the lower bound of  $(1 - o(1)) \euler n \ln(n)$ for the expected optimization time of the (1+1)~EA on \onemax~\cite{DoerrFW10} to all functions in that class (including all linear functions with non-zero coefficients).  

Together with our upper bound, we thus obtain the remarkable result that all linear functions have roughly (within a 39\% range) the same optimization time.

To further demonstrate the strength of multiplicative drift analysis, we give straight-forward analyses for three prominent combinatorial problems. We consider the problems of computing minimum spanning trees (MST), single-source shortest paths (SSSP), and Euler tours. Here, we reproduce the results obtained in~\cite{NeumannW07} (cf.~Theorem~\ref{thm:mst}), in~\cite{BaswanaBDFKN09} (cf.~Theorem~\ref{thm:sssp}), and in~\cite{DoerrJ07} (cf.~Theorem~\ref{thm:euler}), respectively. In doing so, we improve the leading constants of the asymptotic bounds.

\section{Multiplicative Drift Analysis}
\label{sec:drift}

Drift analysis can be used to track the optimization behavior of a randomized search heuristic over a search space by measuring the progress of the algorithm with respect to a \emph{potential function}. Such a function maps each search point to a non-negative real number, where a potential of zero indicates that the search point is optimal.

\begin{theorem} [Additive Drift~\cite{HeY04}]
\label{thm:additivedrift}
Let~$ S\subseteq\mathbb{R}$ be a finite set of positive numbers and let $\{X^{(t)}\}_{t\in\mathbb{N}}$ be a sequence of random variables over~$ S\cup\{0\}$. Let~$T$ be the random variable that denotes the first point in time~$t\in\mathbb{N}$ for which $X^{(t)}=0$. 

Suppose that there exists a constant~$\delta>0$ such that
\begin{equation}
\label{eq:additivedrift}
\EXP\big[X^{(t)}-X^{(t+1)}\bmid T>t\big]\ge \delta
\end{equation}
holds. Then
\begin{equation*}
\EXP\big[T\bmid X^{(0)}\big] \le \frac{X^{(0)}}{\delta}.
\end{equation*}
\end{theorem}

This theorem tells us how to link the expected time at which the potential reaches zero to the first time the expected value of the potential reaches zero. If in expectation the potential decreases in each step by~$\delta$ then after $X^{(0)}/\delta$ steps the expected potential is zero. Thus, one might expect the expected number of steps until the (random) potential reaches zero to be $X^{(0)}/\delta$, too. This is indeed the case in the setting of the previous theorem.

In order to apply the previous theorem to the analysis of randomized search heuristics over a (finite) search space~$\mathcal{S}$, we define a potential function~$h\colon\mathcal{S}\to\mathbb{R}$ which maps all optimal search points to zero and all non-optimal search points to values strictly larger than zero. We choose the random variable~$X^{(t)}$ as the potential~$h(x^{(t)})$ of the search point (or population) in the $t$-th iteration of the algorithm. Then the random variable~$T$ becomes the optimization time of the algorithm, that is, the number of iterations until the algorithm finds an optimum.

When applying Theorem~\ref{thm:additivedrift}, we call the expected difference between $h(x^{(t)})$ and $h(x^{(t+1)})$ the \emph{drift} of the random process~$\{x^{(t)}\}_{t\in\mathbb{N}}$ with respect to~$h$. We say this drift is \emph{additive} if condition~(\ref{eq:additivedrift}) holds.

\subsection{Ideal Potential Functions for Additive Drift Analysis}

The application of additive drift analysis (Theorem~\ref{thm:additivedrift}) to the runtime analysis of randomized search heuristics requires a suitable potential function. The following lemma (Lemma~3 in~\cite{HeY04}) tells us that if the random search points~$x^{(0)},x^{(1)},x^{(2)},\dots$ generated by a search heuristic form a homogeneous absorbing Markov chain, then there always exists a potential function such that condition~(\ref{eq:additivedrift}) in Theorem~\ref{thm:additivedrift} holds with equality; namely the function that attributes to each search point the expected optimization time of the algorithm starting in that point.

\begin{lemma}[\cite{HeY04}]
\label{lem:idealpotential}
Let~$\mathcal{S}$ be a finite search space and~$\{x^{(t)}\}_{t\in\mathbb{N}}$ the search points generated by a homogeneous absorbing Markov chain on~$\mathcal{S}$. Let~$T$ be the random variable that denotes the fist point in time~$t\in\mathbb{N}$ such that~$x^{(t)}$ is optimal.

Then the drift on the potential function~$g\colon\mathcal{S}\to\mathbb{R}$ with
\begin{equation*}
g(x):=\EXP[T\mid x^{(0)}=x]
\end{equation*}
satisfies
\begin{equation*}
\EXP\big[g(x^{(t)})-g(x^{(t+1)})\bmid T>t\big]=1.
\end{equation*}
\end{lemma}

\ignore{
\begin{proof}
By $g(x)=\EXP[T\mid x^{(0)}=x]$ and by the definition of conditional expectation we have
\begin{equation*}
\EXP\big[g(x^{(t)})\bmid T>t\big]=\sum_{x\in\mathcal{S}}\EXP\big[T\bmid x^{(0)}=x]\,\Pr\big[x^{(t)}=x\bmid T>t\big].
\end{equation*}
Since $\Pr[x^{(t)}=x\mid T>t]=0$ for~$g(x)=0$, we have
\begin{equation*}
\EXP\big[g(x^{(t)})\bmid T>t\big]=\sum_{x\in\mathcal{S}:g(x)>0}\EXP\big[T\bmid x^{(0)}=x]\,\Pr\big[x^{(t)}=x\bmid T>t\big].
\end{equation*}
Furthermore, since~$\{x^{(t)}\}$ is a Markov process, we have
\begin{equation*}
\EXP\big[T\bmid x^{(0)}=x\big]=\EXP\big[T\bmid x^{(t)}=x, T>t\big]-t,
\end{equation*}
 for all~$x\in\mathcal{S}$ such that $g(x)\neq 0$ and thus 
\begin{equation*}
\EXP\big[g(x^{(t)})\bmid T>t\big]=\EXP\big[T\bmid T>t\big]-t.
\end{equation*}
by the law of total expectation. Similarly,
\begin{equation*}
\EXP\big[g(x^{(t+1)})\bmid T>t\big]=\EXP\big[T\bmid T>t\big]-(t+1)
\end{equation*}
and the theorem follows.
\end{proof}
}
In a way, $\EXP[T\mid x^{(0)}=x]$ is an ``ideal'' potential function for Theorem~\ref{thm:additivedrift}. It satisfies the additive drift condition~(\ref{eq:additivedrift}) with equality and results in precise upper bound on~$\EXP[T\mid x^{(0)}=x]$. However, the previous theorem is not directly helpful in the runtime analysis of randomized search heuristics. In order to apply the previous theorem, we need to know the exact expected optimization time of a algorithm starting from every point in the search point. But with all this known, Theorem~\ref{thm:additivedrift} does not provide new information.

Still, the previous theorem indicates that potential functions which approximate the expected optimization time in the respective point are good candidates likely to satisfy the additive drift condition. In the next section, we will see such a potential function suitable for the analysis of the optimization behavior of the (1+1)~Evolutionary Algorithm on linear functions.

\subsection{A Multiplicative Drift Theorem}

The drift theorem presented in this subsection can be considered as the multiplicative version of the classical additive result. Since we derive it from the original result, it is clear that the multiplicative version cannot be stronger than the original theorem.

\begin{theorem} [Multiplicative Drift]
\label{thm:multidrift}
Let~$ S\subseteq\mathbb{R}$ be a finite set of positive numbers with minimum $s_{\min}$. Let $\{X^{(t)}\}_{t\in\mathbb{N}}$ be a sequence of random variables over~$ S\cup\{0\}$. Let~$T$ be the random variable that denotes the first point in time~$t\in\mathbb{N}$ for which $X^{(t)}=0$. 

Suppose that there exists a constant~$\delta>0$ such that
\begin{equation}
\label{eq:multidrift}
\EXP\big[X^{(t)}-X^{(t+1)}\bmid X^{(t)}= s\big]\ge \delta s
\end{equation}
holds for all~$s\in S$ with $\Pr[X^{(t)}= s]>0$. Then for all~$s_0\in S$ with~$\Pr[X^{(0)}=s_0]>0$,
\begin{equation*}
\EXP\big[T\bmid X^{(0)}=s_0\big] \le \frac{1+\ln(s_0/s_{\min})}{\delta}.
\end{equation*}
\end{theorem}

Like for the notion of additive drift, we say that the drift of a random process~$\{x^{(t)}\}_{t\in\mathbb{N}}$ with respect to a potential function~$g$ is \emph{multiplicative} if condition~(\ref{eq:multidrift}) holds for the associated random variables~$x^{(t)}:=g(x^{(t)})$.

The advantage of the multiplicative approach is that it allows to use potential functions which are more natural. The most natural potential function, obviously, is the distance of the objective value of the current solution to the optimum. This often is a good choice in the analysis of combinatorial optimization problems. For example, in Section~\ref{sec:combinatorial} we see that the runtimes of the (1+1)~EA on finding a minimum spanning tree, a shortest path tree, or an Euler tour can be bounded by analyzing this potential function.

Another potential function for which drift analysis has been successfully applied is the \emph{distance in the search space} between the current search points and a (global) optimum.

The typical example for this is the drift analysis for linear functions in Section~\ref{sec:linear}, where we use the (weighted) Hamming distance to the optimum as potential for all functions of this class. While being more difficult to analyze, this approach often gives tighter bounds which are independent of range of potential fitness values.

Note that multiplicative drift analysis applies to all situations where previously the so-called \emph{method of expected weight decrease} was used. This method also builds on the observation that if the drift is multiplicative (that is, condition~(\ref{eq:multidrift}) holds), then at time $t=(1+\ln(s_0/s_{\min}))/\delta$ the expected potential~$X^{(t)}$ is at most $s_0/\euler$. Afterwards, various methods (variants of Wald's identity in~\cite{DrosteJW02,Jagerskupper08} and Markov's inequality in~\cite{NeumannW07,BaswanaBDFKN09}) are used to show that the expected stopping time~$\EXP[T]$ is indeed in this regime. However, the bounds obtained in this way are not best possible. This is demonstrated in Section~\ref{sec:jagerskupper} where we replace for the proofs in~\cite{Jagerskupper08} the method of expected weight decrease by the above multiplicative drift theorem. This results in an immediate improvement of the leading constant in the main runtime bound of~\cite{Jagerskupper08}.

\begin{proof}[Proof of Theorem~\ref{thm:multidrift}]
Let~$g\colon S\to\mathbb{R}$ be the function defined by
\begin{equation*}
g(s):=1+\ln\frac{s}{s_{\min}}.
\end{equation*}

Let~$ R:=g(S)$ be the image of~$g$ and let $\{Z^{(t)}\}_{t\in\mathbb{N}}$ be the sequence of random variables over~$ R\cup\{0\}$ given by
\begin{equation*}
Z^{(t)}:=\begin{cases}
0 & \text{if }X^{(t)}=0,\\
g(X^{(t)}) & \text{otherwise.}
\end{cases}
\end{equation*}

Then~$T$ is also the first point in time~$t\in\mathbb{N}$ such that~$Z^{(t)}=0$. Suppose~$T>t$. Then we have~$Z^{(t)}=g(X^{(t)})>0$. If~$X^{(t+1)}=0$, then also~$Z^{(t+1)}=0$ and
\begin{equation}
\label{eq:multzero}
Z^{(t)}-Z^{(t+1)}=1+\ln\Big(\frac{X^{(t)}}{s_{\min}}\Big)\ge 1=\frac{X^{(t)}-X^{(t+1)}}{X^{(t)}}.
\end{equation}
Otherwise, $Z^{(t+1)}=g(X^{(t+1)})$ and again
\begin{equation}
\label{eq:multgeneral}
Z^{(t)}-Z^{(t+1)}=\ln\Big(\frac{X^{(t)}}{X^{(t+1)}}\Big)\ge\frac{X^{(t)}-X^{(t+1)}}{X^{(t)}},
\end{equation}
where the last inequality follows from
\begin{equation*}
\frac{u}{w}=1+\frac{u-w}{w}\le \euler{}^{\frac{u-w}{w}}
\end{equation*}
which implies
\begin{equation*}
\ln\Big(\frac{u}{w}\Big)\le\frac{u-w}{w}
\end{equation*}
and thus
\begin{equation*}
\ln\Big(\frac{w}{u}\Big)\ge\frac{w-u}{w}.
\end{equation*}
for all~$u,w\in\mathbb{R}$.

Hence, by~(\ref{eq:multzero}) and~(\ref{eq:multgeneral}), independent of whether~$Z^{(t+1)}=0$ or~$Z^{(t+1)}\neq 0$, we have
\begin{equation*}
Z^{(t)}-Z^{(t+1)}\ge\frac{X^{(t)}-X^{(t+1)}}{X^{(t)}}.
\end{equation*}

Let~$r\in R$. Since~$g$ is bijective, there exist a unique~$s\in S$ such that $r=g(s)$. Moreover, the events~$Z^{(t)}=r$ and~$X^{(t)}=s$ coincide. Hence, we have by condition~(\ref{eq:multidrift}) that
\begin{equation*}
\EXP[Z^{(t)}-Z^{(t+1)}\mid Z^{(t)}=r]\ge\frac{\EXP[X^{(t)}-X^{(t+1)}\mid X^{(t)}=s]}{s}\ge\delta.
\end{equation*}

Finally, we apply Theorem~\ref{thm:additivedrift} for additive drift and obtain for~$s\in S$ with~$\Pr[X^{(0)}=s]>0$ that
\begin{equation*}
\EXP[T\mid X^{(0)}=s]=\EXP[T\mid Z^{(0)}=g(s)]\le\frac{g(s)}{\delta}\le\frac{1+\ln(s/s_{\min})}{\delta}
\end{equation*}
which concludes the proof of the theorem.
\end{proof}

In Section~\ref{sec:linear} and Section~\ref{sec:combinatorial}, we demonstrate the strength of this new tool by applying it to four well-known problems: the problem of minimizing linear pseudo-Boolean functions, the minimum spanning tree problem, the single-source shortest path problem, and the problem of finding Euler tours.

\section{The Runtime of the (1+1)~Evolutionary Algorithm on Pseudo-Boolean Functions}%
\label{sec:linear}

Many optimization problems can be phrased as the problem of maximizing or minimizing a pseudo-Boolean function $f\colon \{0,1\}^n \to \mathbb{R}$ where~$n$ is a positive integer. In the setting of randomized search heuristics, such a function~$f$ is considered to be a black-box, that is, the optimization process can access~$f$ only by evaluating it at limited number of points in~$\{0,1\}^n$.

\begin{algorithm2e*}[t]
\label{alg:oneoneea}
choose $x^{(0)}\in\{0,1\}^n$ uniformly at random\;
\For{$t=0$ \textbf{\emph{to}} $\infty$}{%
sample $y^{(t)}\in\{0,1\}^n$ by flipping each bit in~$x^{(t)}$ with probability~$1/n$\;
\If{$f(y^{(t)})\le f(x^{(t)})$}{$x^{(t+1)}:=y^{(t)}$}
\Else{$x^{(t+1)}:=x^{(t)}$}
}
\caption{The (1+1)~Evolutionary Algorithm ((1+1)~EA) with mutation rate~$1/n$ for minimizing~$f\colon\{0,1\}^n\to\mathbb{R}$.}
\end{algorithm2e*}

In this section, we analyze the (1+1)~Evolutionary Algorithm ((1+1)~EA) for pseudo-Boolean functions (Algorithm~\ref{alg:oneoneea}). This algorithm follows the neighborhood structure imposed by the hypercube on~$\{0,1\}^n$ where two points are adjacent if they differ by exactly one bit, that is, if their Hamming distance is one.  The (1+1)~EA successively attempts to improve the so-far best search point by randomly sampling candidates over~$\{0,1\}^n$ according to probabilities decreasing with the distance to the current optimum.

The \emph{optimization time} of the (1+1)~EA on a function~$f$ is the random variable~$T$ that denotes the first point in time~$t\in\mathbb{N}$ such that~$f(x^{(t)})$ is minimal.

One elementary linear pseudo-Boolean function for which the optimization time (1+1)~EA has been analyzed (e.g., in~\cite{Muehlenbein1992} and~\cite{DrosteJW02}) is the function $\onemax\colon\{0,1\}^n\to\mathbb{N}$. This function simply counts the number of one-bits in~$x$, that is,
\begin{equation}
\label{eq:onemax}
\onemax(x):=|x|_1=\sum_{i=1}^n x_i.
\end{equation}
Unlike indicated by the name of this function, we are interested in the time the (1+1)~EA needs to find its minimum. Thus, in the selection step (Step~4) of each iteration, the (1+1)~EA accepts the candidate solution~$y^{(t)}$ if and only if the number of bits equal to $1$ does not increase. 

Consider the progress~$\Delta^{(t)}:=\onemax(x^{(t)})-\onemax(x^{(t+1)})$ of the (1+1)~EA in the $t$-th iteration. By construction of the (1+1)~EA, $\Delta^{(t)}$ cannot be negative. By definition, the number of one-bits~$x^{(t)}$ is~$\onemax(x^{(t)})$. For each of these one-bits, there is a $(1/n)(1-1/n)^{n-1}\ge 1/(\euler n)$ chance that only this one-bit is flipped when sampling~$y^{(t)}$, thus increasing the value of~$\onemax(x^{(t)})$ by one. Hence,
\begin{equation*}
\EXP\big[\Delta^{(t)}\bmid x^{(t)}\big]\ge\frac{\onemax(x^{(t)})}{\euler n}.
\end{equation*}

 Thus, multiplicative drift analysis (Theorem~\ref{thm:multidrift}) immediately gives us the well-known result
\begin{equation*}
\EXP[T_{\onemax}]\le\euler n\Big(1+\ln\EXP\big[\onemax(x^{(0)})\big]\Big)=\euler n\Big(1+\ln\Big(\frac{n}{2}\Big)\Big).
\end{equation*}

Another elementary linear pseudo-Boolean function is \binval. This function maps a bitstring to the binary value it represents (where~$x_1$ represents the lowest and~$x_n$ the highest bit).
\begin{equation}
\label{eq:binval}
\binval(x)=\sum_{i=1}^n{2^{i-1} x_i}.
\end{equation}

Again, for $\Delta^{(t)}:=\binval(x^{(t)})-\binval(x^{(t+1)})$, we have
\begin{equation*}
\EXP\big[\Delta^{(t)}\bmid x^{(t)}\big]\ge\frac{\binval(x^{(t)})}{\euler n}
\end{equation*}
and thus
\begin{equation*}
\EXP[T_{\binval}]\le\euler n\Big(1+\ln\EXP\big[\binval(x^{(0)})\big]\Big)=\euler n\Big(1+\ln\Big(\frac{2^n-1}{2}\Big)\Big).
\end{equation*}

Note, that the previous inequality gives us only a quadratic upper bound of~$\O(n^2)$ for the expected optimization time of the (1+1)~EA on~\binval. However, it is known that for all linear functions --- including \binval --- the expected optimization time of the (1+1)~EA is~$\O(n\ln n)$. We discuss this in the following subsections and give a simplified proof using multiplicative drift analysis.

\subsection{Linear Functions}

A classical test problem for the runtime analysis of randomized search heuristics is the minimization of \emph{linear functions}.

Let~$n\in\mathbb{N}$ be a positive integer. A function~$f\colon\{0,1\}^n\to\mathbb{R}$ on~$n$ bits is \emph{linear}, if there exists weights~$w_1,\dots w_n \in \mathbb{R}$ such that
\begin{equation*}
f(x)=\sum_{i=1}^n w_i x_i
\end{equation*}
for all~$x\in\{0,1\}^n$. In~\cite{DrosteJW02} it has been argued and it is easily seen that in the analysis of upper bounds of the expected optimization time of the (1+1)~EA on linear functions we may assume without loss of generality that the weights~$w_i$ are all positive and sorted, that is, 
\begin{equation}
\label{eq:monotone}
0<w_1\le w_2\le\dots\le w_n.
\end{equation}
We simply call such weights \emph{monotone}. Moreover, for the runtime bounds we consider in this work it does not matter whether the (1+1)~EA minimizes or maximizes the linear function. This is true since maximizing a function $f$ is equivalent to minimizing $-f$ and vice versa (for~$-f$ we again have to invoke above argument which allows us to assume monotonicity of the weights).

Thus, from now on, we suppose that every linear function satisfies condition~(\ref{eq:monotone}). Furthermore, we formulate all results for the minimization problem, even if the referenced results originally considered the problem of maximizing linear functions.

We have already seen two prominent examples of linear functions, namely the functions \onemax and \binval. When minimizing \onemax, the (1+1)~EA accepts a new bit string in the selection step (Step~4) if the number of one-bits did not increase. In contrast, when minimizing \binval, the inequality $2^k > \sum_{i=1}^{k-1}{2^{i-1}}$ implies that the (1+1)~EA accepts a new bit string if and only if the highest-index bit that is touched in the mutation step (Step~3) is flipped from one to zero. 

In spite of this difference in behavior, Droste, Jansen and Wegener showed in their seminal paper~\cite{DrosteJW02} that for all linear functions the expected optimization time of the (1+1)~EA is $\Theta(n \log n)$.

\begin{theorem}[\cite{DrosteJW02}]
\label{thm:ealinear}
For all positive integers~$n\in\mathbb{N}$, the expected running time of the (1+1)~EA on the class of linear functions with non-zero weights is $\Theta(n \log n)$.
\end{theorem}

The proof of Droste, Jansen and Wegener applies a level based argument to the potential function (called \emph{artificial fitness function}) $g\colon \{0,1\}^n \to \mathbb{R}$ such that for all~$x\in\{0,1\}^n$
\begin{equation}
\label{eq:droste}
g(x)=\sum_{i=1}^{\lfloor \frac{n}{2}\rfloor}x_i+\sum_{i=\lfloor \frac{n}{2}\rfloor +1}^n 2x_i.
\end{equation}

A much easier proof avoiding partitioning arguments and instead working completely in the framework of drift analysis, was given by He and Yao in~\cite{HeY04}. There, additive drift analysis is applied to the potential function $\widetilde{g}\colon \{0,1\}^n \to \mathbb{R}$ such that
\begin{equation}
\label{eq:heyao}
\widetilde{g}(x)=\ln\Big(1+\sum_{i=1}^{\lfloor \frac{n}{2}\rfloor}x_i+\sum_{i=\lfloor \frac{n}{2}\rfloor +1}^n c x_i\Big)
\end{equation} 
 for all~$x\in\{0,1\}^n$. For this function, with~$1<c\le 2$ chosen arbitrarily, they show that for all~$x\in\{0,1\}^n\setminus\{(0,\dots,0)\}$
\begin{equation*}
\EXP\big[\widetilde{g}(x^{(t)})-\widetilde{g}(x^{(t+1)})\bmid x^{(t)}=x\big]=\Omega(1/n).
\end{equation*}
Afterwards, they apply Theorem~\ref{thm:additivedrift} to show Theorem~\ref{thm:ealinear}. However, while this approach strongly reduced the complexity of the proof in~\cite{DrosteJW02}, introducing the natural logarithm into the potential function still resulted in unnecessary case distinctions and even inconsistencies in an early version of the proof~\cite{HeY01,HeY02AI}.

\subsection{The Drift for Linear Functions is Multiplicative}

In this subsection, we give a simple proof of the fact that the (1+1)~EA optimizes any linear function in expected time $\O(n \log n)$. Our proof is based on the theorem of multiplicative drift (Theorem~\ref{thm:multidrift}). Although proofs for Theorem~\ref{thm:ealinear} are known \cite{DrosteJW02,HeY04,Jagerskupper08}, we present this alternative approach to demonstrate the strength of the multiplicative version of the classical drift theorem. 

In order to apply Theorem~\ref{thm:multidrift} we need a suitable potential function. For this, we choose the function $g\colon \{0,1\}^n \to \mathbb{R}$ such that\footnote{We might as well perform our analysis of~$g$ as defined in~(\ref{eq:droste}). However, our choice of~$g$ does not make the somewhat artificial binary distinction between bits with high and low indices and, thus, seems to be more natural.}

\begin{equation*}
g(x)=\sum_{i=1}^n\Big(1+\frac{i}{n}\Big)x_i
\end{equation*} 
for all~$x\in\{0,1\}^n$. This function defines the potential as the \emph{weighted} distance of the current search point to the optimum (the all-zero string) in the search space. More precisely, it counts the number of one-bits, where each bit is assigned a weight between one and two, such that bits which have higher weight in the objective function~$f$ also have higher weight in~$g$. 

We show that the drift of the (1+1)~EA with respect to~$g$ is multiplicative, that is, that condition~(\ref{eq:multidrift}) holds.

\begin{lemma}
\label{lem:pointwise}
Let~$n\in\mathbb{N}$ be a positive integer. Let~$f\colon\{0,1\}^n\to\mathbb{R}$ be a linear function with monotone weights and let~$g\colon\{0,1\}^n\to\mathbb{R}$ be the potential function with $g(x)=\sum_{i=1}^n(1+i/n)x_i$ for all~$x\in\{0,1\}^n$.

Let~$x\in\{0,1\}^n$ and let~$y\in\{0,1\}^n$ be randomly chosen by flipping each bit in~$x$ with probability~$1/n$. Let $\Delta(x):=g(x)-g(y)$ if $f(y)\le f(x)$ and $\Delta(x)=0$ otherwise. Then
\begin{equation*}
\EXP[\Delta(x)]\ge\frac{g(x)}{4\euler n}. 
\end{equation*}
\end{lemma}

This lemma implies that \emph{at every point in the search space} the drift is at least linear in the current potential value. Thus, the multiplicative drift condition~(\ref{eq:multidrift}) holds and Theorem~\ref{thm:ealinear} follows directly by applying Theorem~\ref{thm:multidrift}.

\begin{proof}[Proof of Lemma~\ref{lem:pointwise}]
Since~$\EXP[\Delta(x)\mid f(y)>f(x)]=0$, we have by the law of total expectation that
\begin{equation}
\label{eq:accepted}
\EXP[\Delta(x)]=\EXP[g(x)-g(y)\mid f(y)\le f(x)]\,\Pr[f(y)\le f(x)].
\end{equation}

Let~$I=\{i\in\{1,\dots,n\}\colon x_i=1\}$. We may distinguish three events (cases).
\begin{itemize}
\item[($C_1$)] There is no index $i\in I$ such that $y_i=0$ and~$f(y)\le f(x)$ holds, that is, $x=y$.
\item[($C_2$)] There is exactly one index $i\in I$ such that $y_i=0$ and~$f(y)\le f(x)$ holds.
\item[($C_3$)] There are at least two different indices $j,\ell\in I$ such that $y_j=0$ and $y_\ell=0$ and~$f(y)\le f(x)$ holds.
\end{itemize}

The only possibility for the event~($C_1$) to hold is if~$x=y$. Therefore,
\begin{equation}
\label{eq:casea}
\EXP[g(x)-g(y)\mid (C_1)]=0.
\end{equation}

Next, suppose the event~($C_3$) holds. By linearity of expectation, we have
\[
\EXP[g(x)-g(y)\mid (C_3)]=\sum_{i=1}^n
\EXP[g(x_i)-g(y_i)\mid (C_3)].
\]

On the one hand, the event~($C_3$) implies that there are (at least) two indices~$j$ and~$\ell$ in~$\{1,\dots,n\}$  for which~$x_j=x_\ell=1$ and~$y_j=y_\ell=1$. Since~$g_j\ge 1$ and~$g_\ell\ge 1$, we have
\[
\sum_{i\in I}\EXP[g(x_i)-g(y_i)\mid (C_3)]\ge 2.
\]

On the other hand, if~$i\in\{1,\dots,n\}\setminus I$ then
\[
\EXP[g(x_i)-g(y_i)\mid (C_3)]=-g_i\Pr[y_1=0\mid(C_3)]\ge -\frac{g_i}{n},
\]
since the condition~($C_3$) does not increase the probability of~$1/n$ that the $y_i=0$. Therefore, since the~$g_i$'s are at most two, we have
\begin{equation}
\label{eq:casec}
\EXP[g(x)-g(y)\mid (C_3)]\ge 2-\frac{1}{n}\sum_{i\notin I}g_i\ge 0.
\end{equation}

Therefore, by the law of total expectation and by~(\ref{eq:accepted}), (\ref{eq:casea}) and~(\ref{eq:casec}), we have
\begin{equation}
\label{eq:exactone}
\EXP[\Delta(x)]\ge\EXP[g(x)-g(y)\mid (C_2)]\,\Pr[(C_2)]
\end{equation}
and can focus on the event ($C_2$).

Suppose that ($C_2$) holds.  For every~$i\in I$, we distinguish two events:
\begin{itemize}
\item[($A_i$)] The $i$-th bit is the only one-bit in~$x$ that flips, none of the zero-bits at the positions larger than~$i$ flips, and $f(y)\le f(x)$ holds.
\item[($B_i$)]  The $i$-th bit is the only one-bit in~$x$ that flips, at least one of the zero-bits at the positions larger than~$i$ that flips, and $f(y)\le f(x)$ holds.
\end{itemize}
We substitute the right side in~(\ref{eq:exactone}) and obtain
\begin{equation}
\label{eq:ab}
\EXP[\Delta(x)]\ge\sum_{i\in I}\EXP[\Delta(x)\mid (A_i)]\,\Pr[(A_i)]+\EXP[\Delta(x)\mid (B_i)]\,\Pr[(B_i)]
\end{equation}

Let~$i\in I$ and suppose that the condition~$(A_i)$ holds. Then we have $y_i=0$ and $y_j=x_j$ for all~$j>i$. For a lower bound on $\EXP[\Delta(x)\mid (A_i)]$, we may suppose that~$x_j=0$ for all~$j<i$ and that every flip of a bit with index~$j<i$ is accepted. Therefore, since~$i\le n$
\begin{equation*}
\EXP[\Delta(x)\mid (A_i)]\ge 1+\frac{i}{n}-\sum_{j=1}^{i-1}\frac{1}{n}\Big(1+\frac{j}{n}\Big)=1+\frac{1}{n}-\frac{i(i-1)}{2n^2}\ge 1-\frac{i-3}{2n}.
\end{equation*}
and thus~$\EXP[\Delta(x)\mid (A_i)]$ is positive. Furthermore, $\Pr[(A_i)]\ge\frac{1}{n}\big(1-\frac{1}{n}\big)^{n-1}$ which is the probability that only the $i$-th bit flips. Hence,
\begin{equation}
\label{eq:ai}
\EXP[\Delta(x)\mid (A_i)]\,\Pr[(A_i)]\ge \frac{1}{n}\Big(1-\frac{1}{n}\Big)^{n-1}\Big(1-\frac{i-3}{2n}\Big).
\end{equation}

Next, suppose that condition~$(B_i)$ holds. Then we have $y_i=0$ and $y_\ell=1$ for all~$\ell\in I\setminus\{i\}$, and there exists a~$j>i$ with~$j\notin I$ such that~$y_j=1$. In order to satisfy~$f(y)\le f(x)$, $w_j=w_i$ has to hold. This implies~$x_\ell=y_\ell$ for all~$\ell\in\{1,\dots,n\}\setminus\{i,j\}$. To see this, recall that the $w_i$'s are monotone and we condition on the event that the $i$-th bit is the only bit that flips from one to zero.

Let~$J(i)=\{j\in\{i+1,\dots,n\}\colon x_j=0\text{ and }w_j=w_i\}$. For~$j\in J(i)$ let~$B_{i,j}$ be the event that $y_i=0$, $y_j=1$, and~$y_\ell=x_\ell$ for~$\ell$ not~$i$ or~$j$. Then
\begin{equation*}
\EXP[\Delta(x)\mid (B_i)]\,\Pr[(B_i)]=\sum_{j\in J(i)}\EXP[\Delta(x)\mid B_{i,j}]\,\Pr[B_{i,j}].
\end{equation*}
We substitute $\EXP[\Delta(x)\mid B_{i,j}]=-\frac{j-i}{n}$ and $\Pr[B_{i,j}]=\frac{1}{n^2}\big(1-\frac{1}{n}\big)^{n-2}$ in the previous equation. Since these conditional expectations are always negative, we may pessimistically assume that $J(i)=\{i+1,\dots,n\}$ and get
\begin{equation*}
\sum_{j\in J(i)}\EXP[\Delta(x)\mid B_{i,j}]=-\frac{(n-i)}{n}\,\frac{(n+1-i)}{2}\ge -\Big(1-\frac{1}{n}\Big)\frac{(n+1-i)}{2}
\end{equation*}
and therefore
\begin{equation}
\label{eq:bi}
\EXP[\Delta(x)\mid (B_i)]\,\Pr[(B_i)]\ge -\frac{1}{n}\Big(1-\frac{1}{n}\Big)^{n-1}\frac{n+1-i}{2n}.
\end{equation}

Finally, we substitute~(\ref{eq:ai}) and~(\ref{eq:bi}) in~(\ref{eq:ab}) and derive
\[
\EXP[\Delta(x)]\ge\frac{1}{n}\Big(1-\frac{1}{n}\Big)^{n-1}\sum_{i\in I}
1-\frac{i-3}{2n}-\frac{n+1-i}{2n}=\frac{1}{n}\Big(1-\frac{1}{n}\Big)^{n-1}\frac{n+2}{4n}\sum_{i\in I}2.
\]
Since~$g_i=1+i/n\le 2$ for all~$i\in I$, we have $\sum_{i\in I}2\ge g(x)$ and therefore
\[
\EXP[\Delta(x)]\ge \frac{g(x)}{4\euler{}n}
\]
which concludes the proof of the lemma.

\end{proof}

\subsection{Distribution-based Versus Point-wise Drift}
\label{sec:jagerskupper}

In this subsection we show an almost tight upper bound on the expected optimization time of the (1+1)~EA on linear functions.

If we take a closer look at Lemma~\ref{lem:pointwise}, we see that it holds \emph{point-wise}, that is, it guarantees
\begin{equation}
\label{eq:pointwise}
\EXP\big[g(x^{(t)})-g(x^{t+1})\bmid x^{(t)}=x\big]\ge\frac{g(x)}{4\euler n}
\end{equation}
for all~$x\in\{0,1\}^n\setminus\{(0,\dots,0)\}$. This is far stronger than the positive \emph{average} drift condition~(\ref{eq:multidrift}) which only requires

\begin{equation}
\label{eq:average}
\EXP\big[g(x^{(t)})-g(x^{t+1})\bmid g(x^{(t)})=s, T>t\big]\ge\delta s
\end{equation}
for all~$s\in\mathbb{R}$ such that $\Pr[g(x^{(t)})=s, T>t]>0$.

The advantage of the stronger point-wise drift assumption is that it immediately guarantees that the result of Theorem~\ref{thm:ealinear} holds for all initial individuals. 

The main reason, however, for not using the weaker condition~(\ref{eq:average}) is that this requires a deeper understanding of the probability distribution of~$x^{(t)}$.

Let us stress that finding a potential function satisfying the stronger point-wise drift condition is usually very tricky. For example, one may ask why not take $\onemax(x)$ as potential function to bound the expected optimization time of the (1+1)~EA for minimizing linear functions. 

However, an easy observation reveals that there is an objective function~$f$ and a search point $x$ such that $g$ yields to small a drift with respect to $f$. To see this, let $x=(x_1,\dots,x_n):=(0,\dots,0,1)$ and let $f:=\binval$ be the function to be minimized. Then the point-wise drift (\ref{eq:pointwise}) with respect to $\onemax$ is only~$1/n^2$. This example shows that finding a potential function yielding point-wise drift for all $x$ and all $f$ may be difficult. This observation is not to be confused with that in the discussion following~(\ref{eq:binval}). There, we determined the drift using the function \binval itself as potential. Here, we use  \onemax, that is, the 1-norm as potential function.

J{\"a}gersk{\"u}pper~\cite{Jagerskupper08} was the first to overcome the difficulties of point-wise drift. While he still avoids completely analyzing the actual distribution of~$x^{(t)}$, he does show the following property of this distribution which in turn allows him to use an average drift approach. In this way, he omits the need for point-wise drift. J\"agersk\"upper's simple observation is that at any time step $t$, the more valuable bits are more likely to be in the right setting (cf. Theorem~1 in \cite{Jagerskupper08}).

\begin{theorem}[\cite{Jagerskupper08}]
\label{thm:jagerskupper}
Let~$n\in\mathbb{N}$ be a positive integer and let $x^{(t)}$ denote the random individual (distributed over $\{0,1\}^n$) after $t\in\mathbb{N}$ iterations of the (1+1)~EA minimizing a linear function $f\colon\{0,1\}^n\to\mathbb{R}$. Then
\begin{equation*}
\Pr[x^{(t)}_1 = 0]\le\dots\le\Pr[x^{(t)}_n = 0].
\end{equation*}
Moreover, for all~$k\in\{0,\dots,n\}$, this statement remains true if we condition on~$\onemax(x)=k$.
\end{theorem}

Using this theorem, he was able to show a lower bound of~${\Omega(1/n)}$ for the drift of \onemax  as potential function for any linear function.

\begin{lemma}[\cite{Jagerskupper08}]
\label{lem:drift}
Let~$n\in\mathbb{N}$ be a positive integer and let~$f\colon\{0,1\}^n\to\mathbb{N}$ be a linear function. Let~$x^{(t)}$ be the individual in the~$t$-th iteration of the (1+1)~EA minimizing~$f$. Then
\begin{equation*}
\EXP[\onemax(x^{(t)})-\onemax(x^{(t+1)})\mid \onemax(x^{(t)})=k]\ge\frac{(\euler -2)k}{\euler n}.
\end{equation*}
holds for all~$k\in\{0,\dots,n\}$ and~$t\in\mathbb{N}$.
\end{lemma}

In addition to a more natural proof of the $O(n \ln(n))$ bound for expected optimization time of the (1+1)~EA minimizing a linear function,  J{\"a}gersk{\"u}pper was able to give a meaningful upper bound on the leading constant (cf. Theorem 2 in \cite{Jagerskupper08}).

\begin{theorem}[\cite{Jagerskupper08}]
\label{thm:upperjagerskupper}
For all positive integers~$n\in\mathbb{N}$, the expected optimization time of the (1+1)~EA minimizing a linear function on~$n$ bits is at most of order $(1+o(1)) 2.02 \euler n\ln(n)$.
\end{theorem}

Using multiplicative drift analysis (Theorem~\ref{thm:multidrift})  on the result of Lemma~\ref{lem:drift} and thus replacing the halving argument employed by J{\"a}gersk{\"u}pper for the proof of Theorem~\ref{thm:upperjagerskupper}, the constant of $2.02 \euler$ in the upper bound of the previous theorem instantly improves to $1.39 \euler$. In the light of our lower bound of $1.00\euler$, to be proven in the next subsection, this is a considerable progress. 

\begin{theorem}
\label{thm:upperbound}
For all positive integers~$n\in\mathbb{N}$, the expected optimization time of the (1+1)~EA minimizing a linear function on~$n$ bits is at most of order $(1+o(1))\frac{\euler}{\euler -2}n\ln(n)\approx (1+o(1)) 1.39 \euler n\ln(n)$.
\end{theorem}

\subsection{The (1+1)~EA Optimizes OneMax Faster than any Function with a Unique Global Optimum}

In this section, we show that the expected optimization time of (1+1)~EA on any pseudo-Boolean function with a unique global optimum is at least as large as its expected optimization time on the basic function \onemax. In particular, this is true for every linear function with non-zero coefficients.

In other words, if a function is easier to optimize than \onemax, then this can only be due to the fact that it has more than one global optimum.
The general lower bound then follows from the following theorem by Doerr, Fouz and Witt \cite{DoerrFW10}, which provides a lower bound for \onemax.

\begin{theorem}[\cite{DoerrFW10}]
\label{thm:LBOnemax}
For all positive integers~$n\in\mathbb{N}$, the expected optimization time of the (1+1)~EA minimizing \onemax on~$n$ bits is at least~$(1-\o(1)) \euler n\ln(n)$.
\end{theorem}

Thus, it remains to show that \onemax is optimized fastest. The result itself was announced by Scheder and Welzl \cite{SchederW08}. Their idea to prove this statement, however, differs from the one given below.

\begin{theorem}
\label{thm:lowerbound}
Let~$n\in\mathbb{N}$ be a positive integer. The expected optimization time of the (1+1)~EA on any function~$f\colon\{0,1\}^n\to\mathbb{R}$ that has a unique global optimum is as least as large as its expected optimization time on \onemax.
\end{theorem}

The theorem can be formalized as follows: Let $f$ be a function with a unique global optimum. Let $\{x^{(t)}\}_{t \in \mathbb{N}}$ be the search points generated by the (1+1)~EA minimizing~$f$. Let $T_f:=\min\{t \in \mathbb{N}\mid f(x^{(t)}) =0\}$ be the optimization time of the (1+1)~EA on $f$. Then $\EXP[T_f] \geq \EXP[T_{\onemax}].$

Theorems~\ref{thm:LBOnemax} and~\ref{thm:lowerbound} immediately yield the following.

\begin{corollary}
\label{cor:LBLinear}
For all positive integers~$n\in\mathbb{N}$, the expected optimization time of the (1+1)~EA minimizing a function with a unique global optimum on~$n$ bits is at least $(1-\o(1))  \euler n\ln(n)$.
\end{corollary}

For the proof of Theorem~\ref{thm:lowerbound} we first show a preliminary lemma. It formalizes the following intuition. Let $x$ and $\tilde{x}$ be two search points such that~$|x|_1\le|\tilde{x}|_1$. Then the probability that the (1+1)~EA samples a new search point with exactly $j<|x|_1$ one-bits from~$x$ is at least as big as from~$\tilde{x}$.

\begin{lemma}
\label{lem:probabilities}
Let~$n\in\mathbb{N}$ with $n \ge 1$. Let $x,\tilde{x} \in \{0,1\}^n$ with $|x|_1 \le |\tilde{x}|_1$. Let $y$ and $\tilde{y}$ two random points in $\{0,1\}^n$ obtained from $x$ and~$\tilde{x}$ by independently flipping with probability~$1/n$ each bit of~$x$ and~$\tilde{x}$, respectively.

Then for every $j\in\{0,\dots,|x|_1-1\}$, 
\begin{equation*}
	\Pr\big[ |y|_1 = j\big] \ge \Pr\big[ |\tilde{y}|_1 = j\big] .
\end{equation*}
\end{lemma}

\begin{proof}
Let $k:=|x|_1$. The lemma holds trivially if $|\tilde{x}|_1=k$. 

Suppose that $|\tilde{x}|_1=k+1$. Then
\begin{equation*}
	\Pr\big[ |y|_1 = j\big]=\hspace{-1em}
	\sum_{i=0}^{\min\{j,n-k\}}\hspace{-0.4em}\binom{k}{j-i} \binom{n-k}{i}\Big(\frac{1/n}{1-1/n}\Big)^{k-j+2i}\big(1-1/n\big)^n
\end{equation*}
and
\begin{equation*}
	\Pr\big[ |\tilde{y}|_1 = j\big]=\hspace{-1em}
	\!\sum_{i=0}^{\min\{j,n-k-1\}}\hspace{-0.4em}\binom{k+1}{j-i} \binom{n-k-1}{i}\Big(\frac{1/n}{1-1/n}\Big)^{k+1-j+2i}\big(1-1/n\big)^n.
\end{equation*}
As all summands in the previous two equations are positive, it suffices to see that the quotient
\begin{equation*}
\label{ineq:binom}
\frac{\binom{k}{j-i} \binom{n-k}{i}\big(\frac{1/n}{1-1/n}\big)^{k-j+2i}}{\binom{k+1}{j-i} \binom{n-k-1}{i}   \big(\frac{1/n}{1-1/n}\big)^{k+1-j+2i}}
=\frac{(k+1-j+i)(n-k)(n-1)}{(k+1)(n-k-i)}
\end{equation*}
is minimal for~$i=0$ and~$j=k-1$ and therefore at least~1 for all values $0\leq i \leq \min\{j,n-k-1\}$.

Thus, for $|\tilde{x}|_1=k+1$ the lemma also holds. Finally, for $|\tilde{x}|_1>k+1$, the lemma follows by induction based on the case~$|\tilde{x}|_1=k+1$.
\end{proof}

To prove the main result of this section, Theorem~\ref{thm:lowerbound}, we need some additional notation. Let $f$ be a function with a unique global optimum $x^*$. Without loss of generality, we may assume that~$x^*:=(0,\dots,0)$ is the unique minimum of~$f$. This is justified by the observation that the (1+1)~EA treats the bit-values $0$ and $1$ symmetrically, that is, we might reinterpret one-bits in~$x^*$ as zero-bits without changing the behavior of the algorithm.

Let $\mu(x) := \EXP[T_{\onemax} \mid x^{(0)}=x]$ and $\widetilde{\mu}(x) := \EXP[T_f \mid x^{(0)}=x]$ be the expected optimization times of the (1+1)~EA starting in the point $x$ and minimizing  \onemax and $f$, respectively.

For every $k \in\{0\dots,n\}$ let 
\[
\mu_k := \min \{\mu(x) \mid x\in\{0,1\}^n, |x|_1 = k\}
\]
be the optimization time of the (1+1)~EA optimizing \onemax starting in a point with exactly $k$ one-bits. 

Furthermore, let
\[
\widetilde{\mu}_k := \min \{\widetilde{\mu}(x) \mid x\in\{0,1\}^n, |x|_1 \geq k\}
\]
be the minimum optimization time of the (1+1)~EA minimizing $f$ and starting in a point $x$ with at \emph{least} $k$ one-bits (note the difference to $\mu_k$).

Note that, due to the symmetry of the function \onemax, $\mu_k = \mu(x)$ for every $x\in\{0,1\}^n$ with exactly $k$ one-bits.

\begin{proof}[Proof of Theorem~\ref{thm:lowerbound}]
We inductively show for all $k$ that $\mu_k\le\widetilde{\mu}_k$.
Clearly $\mu_0=0=\widetilde{\mu}_0$.  Therefore, let $k \in \{0,\dots,n-1\}$ and suppose that $\mu_i\le\widetilde{\mu}_i$ for all $i\leq k$.

Let~$x\in\{0,1\}^n$ with~$|x|_1=k+1$ be arbitrary  and let~$y\in\{0,1\}^n$ be a random point generated by flipping each bit in~$x$ independently with probability~$1/n$.

The (1+1)~EA minimizing \onemax and starting in $x$ accepts~$y$ in the selection step (Step~4) if and only if $|y|_1 \leq |x|_1$. Furthermore, we have $\mu(x)=\mu(y)$ if $|y|_1=|x|_1=k+1$. Thus,
\begin{equation*}
	\mu(x) = 1
	+  \mu(x)\Pr\big[|y|_1\ge k+1\big]
	+ \sum_{j=0}^k\EXP\big[\mu(y)\bmid |y|_1=j\big]\,\Pr\big[|y|_1=j\big]
\end{equation*}
and therefore
\begin{equation}
\label{eq:muonemax}
	\mu_{k+1} = 1
	+ \mu_{k+1} \Pr\big[|y|_1\ge k+1\big]
	+ \sum_{j=0}^k \mu_j\Pr\big[|y|_1=j\big].
\end{equation}

Next, let $\tilde{x}\in \{0,1\}^n$ be chosen arbitrarily such that $|\tilde{x}|_1\ge k+1$ and $\widetilde{\mu}_{k+1} = \widetilde{\mu}(\tilde{x})$. Furthermore, let~$\tilde{y}\in\{0,1\}^n$ be a random point generated by flipping each bit in~$\tilde{x}$ independently with probability~$1/n$. Let~$\tilde{z}=\tilde{y}$ if~$f(\tilde{y})\le f(\tilde{x})$ and~$\tilde{z}=\tilde{x}$ otherwise.

Then
\begin{align*}
	\widetilde{\mu}(\tilde{x}) = 1
	&+ \EXP\big[\widetilde{\mu}(\tilde{z})\bmid |\tilde{z}|_1\ge k+1\big]\,\Pr\big[|\tilde{z}|_1\ge k+1\big]\\
        &+ \sum_{j=0}^k\EXP\big[\widetilde{\mu}(\tilde{z})\bmid |\tilde{z}|_1=j\big]\,\Pr\big[|\tilde{z}|_1=j\big] 
\end{align*}
and therefore, by definition of~$\widetilde{\mu}_j$,
\begin{equation}
\label{eq:muf}
	\widetilde{\mu}_{k+1}\ge 1
	+ \widetilde{\mu}_{k+1}\Pr\big[|\tilde{z}|_1\ge k+1\big]
	+ \sum_{j=0}^k \widetilde{\mu}_j\Pr\big[|\tilde{z}|_1=j\big].
\end{equation}

Now, for all~$0\le j\le k$, we have
\begin{equation*}
\Pr\big[|\tilde{z}|_1=j\big]\le\Pr\big[|\tilde{y}|_1=j\big]\le\Pr\big[|y|_1=j\big].
\end{equation*}
The first inequality holds, since the event~$|\tilde{z}|_1=j$ implies the event~$|\tilde{y}|_1=j$. The second inequality follows from Lemma~\ref{lem:probabilities}, since~$|x|_1=k+1 \le |\tilde{x}|_1$.

Considering this relation and the fact that the $\widetilde{\mu}_i$ are monotonically increasing in $i$, we obtain from~(\ref{eq:muf}) that
\begin{equation*}
	\widetilde{\mu}_{k+1}\ge 1
	+ \widetilde{\mu}_{k+1}\Pr\big[|y|_1\ge k+1\big]
	+ \sum_{j=0}^k \widetilde{\mu}_j\Pr\big[|y|_1=j\big].
\end{equation*}

Therefore, the induction hypothesis yields that
\begin{equation*}
	\widetilde{\mu}_{k+1}\ge 1
	+ \widetilde{\mu}_{k+1}\Pr\big[|y|_1\ge k+1\big]
	+ \sum_{j=0}^k\mu_j\Pr\big[|y|_1=j\big] .
\end{equation*}

We subtract both sides of equation~(\ref{eq:muonemax}) from the previous inequality and immediately get $\widetilde{\mu}_{k+1}\ge \mu_{k+1}$ which concludes the induction.

Thus, for all $x \in \{0,1\}^n$, we have
\begin{equation*}
	\mu(x)=\mu_{|x|_1}\le\widetilde{\mu}_{|x|_1}\le\widetilde{\mu}(x).
\end{equation*}
Consequently,~$\EXP[ T_{\onemax}]\le\EXP[T_f]$ holds.
\end{proof}

\ignore{
\subsection{Experiments}

As lower order terms hidden in $(1-o(1))$ factors may lead to a completely different behavior for small instances, we ran a few simple experiments to complement our theoretical results given in the last section. 

\begin{figure}[b]
\begin{center}
\includegraphics[width=0.8\textwidth]{graph.jpg}
\caption{Optimization times of the (1+1)~EA for different linear functions. The second graph depicts a zoom into the results for $700 \leq n \leq 1,000$.}
\label{fig:runtimes}
\end{center}
\end{figure}

For the problems size $n$ varying from $20$ to $1000$ in steps of $20$, we let the (1+1)~EA minimize three different linear functions: \onemax as defined in~(\ref{eq:onemax}), $\binval$ as defined in~(\ref{eq:binval}), and random linear functions. For the latter, for each run independently we chose the coefficients of $f$ uniformly at random from $(0,1]$. For each instance, we repeated the experiment $1000$ times and computed the average optimization time. The results are depicted in Figure~\ref{fig:runtimes}.

As can be seen, in fact, \onemax seems to be the easiest of the three function classes. At the same time, $\binval$ has the largest optimization time. 
The average time difference between the named two functions varies between 4 and 10$\%$, clearly below the 39$\%$ guarantee given by the theoretical analysis. Since we know that the $(1-o(1))e n \ln(n)$ bound for \onemax is sharp, this suggests that the upper bound for general linear functions, in spite of our improvements, is still too pessimistic. 

\clearpage
}
\section{Multiplicative Drift on Combinatorial Problems}
\label{sec:combinatorial}

So far, we have seen that multiplicative drift analysis can be used to simplify the runtime analysis of the (1+1)~EA on linear pseudo-Boolean functions while producing sharper bounds. In this section, we see that optimization processes with multiplicative drift occur quite naturally in combinatorial optimization, too. We demonstrate this claim on two prominent examples, the minimum spanning tree problem and the single source shortest path problem.

\subsection{The Minimum Spanning Tree Problem}
\label{subsec:MST}
In this subsection, we consider the minimum spanning tree (MST) problem analyzed in~\cite{NeumannW07}. Let~$G=(V,E)$ be a connected graph with $n$ vertices, $m$ edges~$e_1,\dots,e_m$, and positive integer edge weights $w_1,\dots,w_m$. In~\cite{NeumannW07}, a spanning tree is represented by a bit string~$x\in\{0,1\}^m$ with~$x_i=1$ marking the presence of the edge~$e_i$ in the tree. 

The fitness value of such a tree is defined by~$w(x)=\sum_{i=1}^m w_i x_i+p(x)$, with~$p(x)$ being a penalty term ensuring that once the (1+1)~EA has found a spanning tree it does no longer accept bit strings that do not represent spanning trees (a new bit-strings is accepted if the fitness value decreases). The minimum weight of a spanning tree is denoted by~$w_{\text{opt}}$ and the maximal edge weight by~$w_{\text{max}}$.

In Lemma~1 of~\cite{NeumannW07}, Neumann and Wegener derive from~\cite{Kano87} the following statement.

\begin{lemma}[\cite{NeumannW07}]
\label{lem:mst}
Let $x\in\{0,1\}^m$ be a search point describing a non-minimum spanning tree. Then there exist a~$k\in\{1,\dots,n-1\}$ and~$k$ different accepted 2-bit flips such that the average weight decrease of these flips is at least $(w(x)-w_{\text{opt}})/k$.
\end{lemma}

Multiplicative drift analysis now gives us a reasonably small constant in the upper bound of the expected optimization time of the (1+1)~EA on the MST problem.

\begin{theorem}
\label{thm:mst}
The expected optimization time of the (1+1)~EA on the MST problem starting with an arbitrary spanning tree of a non-empty graph is at most~$2\euler m^2(1+\ln m+\ln w_{\text{max}})$.
\end{theorem}

\begin{proof}
For all~$t\in\mathbb{N}$, let $x^{(t)}$ be the search point of the (1+1)~EA for the MST problem at time~$t$ and let $X^{(t)}=w(x^{(t)})-w_{\text{opt}}$. Then
\begin{equation*}
X^{(t)}-X^{(t+1)}=w(x^{(t)})-w(x^{(t+1)}).
\end{equation*}

Now, let~$t\in\mathbb{N}$ and~$x\in\{0,1\}^n\setminus\{(0,\dots,0)\}$ be fixed. Let the points $y_{(1)},\dots,y_{(k)}$ with~$k\in\{0,\dots,n-1\}$ be the~$k$ distinct search points in~$\{0,1\}^m$ generated from~$x$ by the $k$ different 2-bit flips according to Lemma~\ref{lem:mst}. That is, we have $w(y_{(i)})\le w(x)$ for all~$i\in\{1,\dots,k\}$ and
\begin{equation}
\label{eq:mstaverage}
\sum_{i=1}^k\big(f(x)-f(y_i)\big)\ge w(x)-w_{\text{opt}}.
\end{equation}

Since the~$y_{(i)}$'s are each generated from~$x$ by a 2-bit flip, we have
\begin{equation}
\label{eq:mstprob}
\Pr\big[x^{(t+1)}=y_{(i)}\bmid x^{(t)}=x\big]=\Big(1-\frac{1}{m}\Big)^{m-2}\Big(\frac{1}{m}\Big)^2
\end{equation}
for all~$i\in\{1,\dots,k\}$. Furthermore
\begin{equation}
\label{eq:mstexp}
\EXP\big[X^{(t)}-X^{(t+1)}\bmid x^{(t)}=x,x^{(t+1)}=y_{(i)}\big]=w(x)-w(y_{(i)})
\end{equation}
holds for all~$i\in\{1,\dots,k\}$.

The (1+1)~EA never increases the current $w$-value of a search point, that is, $X^{(t)}-X^{(t+1)}$ is non-negative. Thus, we have by~(\ref{eq:mstprob}) and~(\ref{eq:mstexp}) that
\begin{equation*}
\EXP\big[X^{(t)}-X^{(t+1)}\bmid x^{(t)}=x\big]\ge\sum_{i=1}^k\Big(w(x)-w(y_{(i)})\Big)\Big(1-\frac{1}{m}\Big)^{m-2}\Big(\frac{1}{m}\Big)^2
\end{equation*}
and therefore, by inequality~(\ref{eq:mstaverage}), we have for all~$x\in\{0,1\}^m$ that
\begin{equation*}
\EXP\big[X^{(t)}-X^{(t+1)}\bmid x^{(t)}=x\big]
\ge
\frac{w(x)-w_{\text{opt}}}{\euler m^2}.
\end{equation*}
In other words,
\begin{equation*}
\EXP[X^{(t)}-X^{(t+1)}\mid X^{(t)}]\ge \frac{X^{(t)}}{\euler m^2}
\end{equation*}
and the theorem follows from the Theorem~\ref{thm:multidrift} with $1\le X^{(t)}\le m w_{\text{max}}$.
\end{proof}

\subsection{The Single-source Shortest Path Problem}
\label{subsec:SSSP}
In~\cite{BaswanaBDFKN09}, Baswana, Biswas, Doerr, Friedrich, Kurur, and Neumann study an evolutionary algorithm that solves the single-source shortest path (SSSP) problem on a directed graph with~$n$ vertices via evolving a shortest-path tree. In the analysis of the upper bound for the expected optimization time, the authors introduce the \emph{gap} $g_i$ in iteration~$i$ as the difference in fitness between the current shortest-path tree candidate and an optimal shortest-path tree.

For every vertex in the tree, its \emph{weight} in the tree is defined as the sum over the weights of edges in the paths leading to the root vertex, or as the penalty term~$n w_{\text{max}}$ if the vertex is not connected to the root. The fitness of a shortest-path tree candidate is then the sum over the weights of all vertices in the tree. Thus the maximal gap is~$n^2 w_{\text{max}}$. In Lemma~1 of~\cite{BaswanaBDFKN09}, the authors then provide the following statement.

\begin{lemma}[\cite{BaswanaBDFKN09}]
Let~$g_i$ denote the gap after~$i$ mutations. Then it holds for the conditional expectation~$\EXP[g_{i+1}\mid g_i=g]$ that
\begin{equation*}
\EXP[g_{i+1}\mid g_i=g]\le g \Big(1-\frac{1}{3\cdot  n^3}\Big).
\end{equation*}
\end{lemma}

To this, we can directly apply Theorem~\ref{thm:multidrift}, taking the gap as a potential. We obtain the following result with a precise constant for the upper bound.

\begin{theorem}
\label{thm:sssp}
The expected optimization time of the (1+1)~EA in~\cite{BaswanaBDFKN09} on the SSSP problem starting with an arbitrary shortest-path tree candidate is at most~$6 n^3 (1+2 \ln n+\ln w_{\text{max}})$.
\end{theorem}

\subsection{The Euler Tour Problem}
\label{subsec:Euler}

The Euler tour problem is to find a Euler tour (a closed walk that visits every edge exactly once) in an input graph which permits such a tour. 

In \cite{DoerrJ07}, possible variants of the (1+1)~EA for the Euler tour problem are analyzed. For the variant using the so-called edge-based distribution on cycle covers, the search space is given by \emph{adjacency list matchings}, where each matching represents a cover of the input graph with edge-disjoint cycles. The fitness of a matching is given by the total number of cycles in the cover. Thus, a fitness of one implies that the graph is covered by a single cycle --- an Euler tour. 

Finding such a tour is then a minimization problem over this search space. For this setup, the following statement is implicitly shown in the proof of Theorem~3.

\begin{lemma}
In a single iteration of the (1+1)~EA in~\cite{DoerrJ07} for the Euler tour problem using the edge-based distribution and starting with an arbitrary cycle cover, the probability to decrease the fitness~$f(x)$ of the current search point~$x$ by one (provided it was not minimal before) is at least~$f(x)/\euler{}m$ where $m$ is the number of edges of the input graph.
\end{lemma}

If we set the fitness minus one as potential, this lemma immediately implies that the expected drift is at least~$f(x)/\euler{}m$. Moreover, the starting potential is at most $m/3$ (each tour has hat least three edges). Again, we can apply Theorem~\ref{thm:multidrift} and reproduce the upper bound the expected optimization time, specifying the leading constant in the process.

\begin{theorem}
\label{thm:euler}
The expected optimization time of the (1+1)~EA in~\cite{DoerrJ07} for the Euler tour problem using the edge-based distribution and starting with an arbitrary cycle cover is at most $\euler{}m\ln m$, where $m$ is the number of edges in the input graph.
\end{theorem}

\section{Discussion and Outlook}%
\label{sec:Discussion}

In this work, we showed that the multiplicative drift condition (\ref{eq:multidrift}) occurs naturally in the runtime analysis of the (1+1)~EA for number of prominent optimization problems (linear functions, minimum spanning trees, shortest paths, and Euler tours). In such situations our multiplicative drift theorem (Theorem~\ref{eq:multidrift}) yields good runtime bounds.

We applied this new tool to various settings. First, we used it to gain new insight in the classical problem of how the (1+1)~EA optimizes linear functions. 

We presented a simplified proof of the, by now, well-known fact that the (1+1)~EA with mutation probability~$1/n$ optimizes any linear function in time $O(n \log n)$. Moreover, we applied our result to the distribution-based drift analysis of J\"agersk\"upper and obtained a new upper bound of $(1+o(1)) 1.39 \euler n \ln(n)$ for the expected optimization time of the (1+1)~EA for arbitrary linear functions.

We complement this upper bound by a lower bound of $(1-o(1)) \euler n \ln(n)$. To do so, we showed that \onemax is the function easiest optimized by the (1+1)~EA. By this we extended a recent lower bound of $(1-o(1)) \euler n \ln(n)$ for the expected optimization time on \onemax to all functions having a unique global optimum. 

Our upper and lower bounds for the expected optimization times of the (1+1)~EA on arbitrary linear functions are relatively close. This raises the question if possibly all linear functions have the same expected optimization time of $(1+o(1)) \euler n \ln(n)$. 

Finally, we reviewed previous runtime analyses of the (1+1)~EA on the combinatorial problems of finding a minimum spanning tree, shortest path tree, or Euler tour in a graph. For all three cases, we exhibited the appearance of multiplicative drift and determined the leading constants in the bounds of the expected optimization times.

In the light of these natural occurrences of multiplicative drift, we are optimistic to see applications of multiplicative drift analysis in the near future.

\subsection*{Acknowledgments}%
We like to thank Dirk Sudholt for pointing out that Theorem~\ref{thm:lowerbound} holds for all pseudo-Boolean function with a unique global optimum rather than only for all linear functions.

\subsection*{Note Added in Proof}
Recently, Doerr and Goldberg~\cite{DoerrG10} have shown that in Theorem~\ref{thm:multidrift}, the stopping time~$T$ is with high probability at most of the same order as the upper bound on its expectation given in inequality~(\ref{eq:multidrift}), if $X^{(0)}$ is at least $\Omega(n)$. Thus, the implicit upper bound given in Theorem~\ref{thm:upperbound} and the bounds in Theorem~\ref{thm:mst} and Theorem~\ref{thm:sssp} also hold with high probability, if we allow a slightly larger leading constant.

%}% sloppypar
\newcommand{\etalchar}[1]{$^{#1}$}

%\bibliographystyle{alpha}
%\bibliography{abbreviations,proceedings,paper}

\end{document}